\documentclass{article}
\usepackage[final]{colm2026_conference}
\usepackage{microtype}
\usepackage{hyperref}
\usepackage{url}
\usepackage{booktabs}
\usepackage{lineno}
\definecolor{darkblue}{rgb}{0, 0, 0.5}
\hypersetup{colorlinks=true, citecolor=darkblue, linkcolor=darkblue, urlcolor=darkblue}
\usepackage{graphicx}
\usepackage{subcaption}
\usepackage{amsmath}
\usepackage{amssymb}
\usepackage{mathtools}
\usepackage{amsthm}
\usepackage{multirow}
\usepackage{makecell}
\usepackage{tcolorbox}
\usepackage{algorithm}
\usepackage{algorithmic}
\usepackage{paralist}
\usepackage{bm}
\usepackage[capitalize,noabbrev]{cleveref}
\usepackage[textsize=tiny]{todonotes}

\theoremstyle{plain}
\newtheorem{theorem}{Theorem}[section]

\newtheorem{corollary}[theorem]{Corollary}
\theoremstyle{definition}
\newtheorem{definition}[theorem]{Definition}
\newtheorem{assumption}[theorem]{Assumption}
\theoremstyle{remark}
\newtheorem{remark}[theorem]{Remark}

\newcommand{\reference}{\ensuremath{\text{ref}}}

\title{In-Context Learning as Implicit Policy Gradient}

\author{Masahiro Kaneko \quad Timothy Baldwin \\
        MBZUAI \\
        {\tt \{masahiro.kaneko,timothy.baldwin\}@mbzuai.ac.ae}
}

\begin{document}
\ifcolmsubmission
\linenumbers
\fi
\maketitle

\begin{abstract}
Recent work has shown that large language models (LLMs) can iteratively improve their outputs by incorporating generated samples and their corresponding evaluation scores as in-context examples.
Despite these empirical findings, the theoretical foundations underlying this phenomenon remain poorly understood.
In this paper, we show that score-conditioned In-Context Learning (ICL) admits a structural correspondence to policy gradient optimization.
We first provide a constructive proof that self-attention mechanisms can implement reward-weighted aggregation analogous to the REINFORCE algorithm under specific weight matrix configurations, and discuss the relationship between this construction and the behavior of pretrained transformers.
The correspondence is directional in hidden-state space and holds exactly only under the stated simplifying conditions; we quantify its strength empirically.
Within our simplified hidden-state model, we furthermore derive an exact upper bound on the distribution shift induced by a bounded attention update, yielding a trust-region-like analogy to KL-constrained policy optimization.
We validate our theory through extensive experiments across multiple LLMs, demonstrating that LLMs effectively utilize score information to shift output distributions toward high-scoring exemplars, and that attention weights exhibit a strong correlation with example scores.
\end{abstract}

\section{Introduction}

Large language models (LLMs) have demonstrated a remarkable ability to adapt to new tasks through In-Context Learning (ICL), where the model learns from examples provided in the prompt without any parameter updates \citep{brown2020language}.
This emergent capability has attracted significant interest from both theoretical and practical perspectives.
On the theoretical side, it has been shown that the forward pass of transformers can implicitly implement gradient descent \citep{vonoswald2023transformers} and ridge regression \citep{akyurek2023learning}.
Additionally, perspectives framing ICL as Bayesian inference over latent concepts acquired during pretraining \citep{xie2022explanation} and mechanistic analyses of the role of induction heads \citep{olsson2022induction} have been proposed.
However, these analyses share a common limitation: they focus exclusively on ICL in supervised learning settings with diverse input--output pairs $(x, y)$.

On the empirical side, it has been demonstrated that LLMs can improve their outputs by conditioning on self-generated samples along with their evaluation scores as in-context examples.
This method conditions the model on multiple output-score pairs $(y, r)$ generated for a fixed input $x$ rather than learning from input-output pairs $(x, y)$.
For example, the ``Optimization by PROmpting'' (OPRO) approach of \citet{yang2023large} provides the model with previously generated outputs and their scores as context, achieving strong results on prompt optimization and combinatorial problems.
Similarly, the self-rewarding mechanism of \citet{yuan2024self} uses self-generated responses with self-assigned scores as in-context examples, enabling continuous improvement of alignment without human annotation.

These successes raise a fundamental theoretical question: what learning algorithm, if any, does score-conditioned generation implement?
The existing theoretical frameworks for ICL cannot answer this question, as they analyze supervised learning with diverse $(x, y)$ pairs rather than reward-based learning with $(y, r)$ pairs for a given input.
This gap leaves us without a principled understanding of when and why score-conditioned ICL works.
Just as $(x, y)$ pairs with direct target supervision naturally connect to supervised learning, $(y, r)$ pairs with scalar reward signals suggest a connection to reinforcement learning, where agents improve their behavior based on such signals.

In this paper, we formalize this intuition by showing that score-conditioned ICL admits a structural correspondence to policy gradient optimization.
Our central claims are threefold.
\textbf{First, attention mechanisms can implement reward-weighted aggregation analogous to REINFORCE}.
When LLMs process score-annotated examples $(y, r)$ through their attention mechanisms, under specific weight matrix configurations, attention weights the examples by their scores rather than by input similarity.
Following the approach of \citet{vonoswald2023transformers}, we provide a constructive proof that this reward-weighted aggregation is structurally analogous to the REINFORCE policy gradient estimator, in the sense that both produce score-weighted updates toward high-reward outputs in the hidden state space.
We emphasize that this is an existence result: it demonstrates that the transformer architecture is \emph{capable} of implementing such a mechanism, and we provide complementary empirical evidence in Appendix~\ref{app:weight_analysis} that pretrained models exhibit functionally similar behavior.

\textbf{Second, the bounded hidden-state update induces a bounded KL shift}.
In explicit policy optimization, KL regularization is widely employed to prevent drastic distributional changes \citep{schulman2017proximal,schulman2015trust}.
In the simplified output-layer model used in our analysis, we prove that the attention-induced hidden-state update yields an exact pathwise upper bound on the distribution shift from the reference model.
This is a trust-region-like property, rather than a claim that ICL explicitly optimizes a KL-penalized objective.

\textbf{Third, we validate our theoretical predictions through extensive experiments across multiple LLMs}.
Through experiments on Llama-3, Olmo-3, Qwen-3, GPT-4o, Gemini 2.5, and Claude Sonnet 4~\citep{anthropic2025claude4,comanici2025gemini,hurst2024gpt,dubey2024llama,olmo2025olmo,qwen3technicalreport}, we confirm: (1) example scores are positively rank-correlated with probability changes; (2) attention weights correlate with scores, providing mechanistic evidence for score-weighted aggregation; and (3) ICL and explicit REINFORCE fine-tuning shift the output distribution in similar directions given identical samples.

\section{Preliminaries}

\subsection{Notation}

We consider the setting where an LLM generates outputs for a given input and receives scalar feedback.
Let $x$ denote the input prompt, $\bm{y} \in \mathbb{R}^D$ denote the embedding representation of a generated output, and $r \in \mathbb{R}$ denote its associated score (e.g., task performance, human preference rating, or model-based evaluation).
We represent a score-annotated example as a combined token $\bm{e} = (\bm{y}, r) \in \mathbb{R}^{D+1}$, concatenating the output embedding with its scalar score.
A context of $N$ such examples generated for a fixed input $x$ is denoted $C = \{(\bm{y}_i, r_i)\}_{i=1}^{N}$.
We use $\pi_{\reference}$ to denote the reference (base) LLM distribution and $\pi_\theta$ for a policy with parameters $\theta$.
The distribution induced by conditioning on context $C$ is written as $\pi(\cdot|x, C)$.

\subsection{Self-Attention}

Self-attention is the core mechanism by which transformers aggregate information across context.
It computes a weighted combination of context elements, where the weights determine how much each element contributes to the output.

\begin{definition}[Linear Self-Attention]
Given embeddings $\bm{E} \in \mathbb{R}^{d \times n}$ and weight matrices $\bm{W}^Q, \bm{W}^K, \bm{W}^V, \bm{W}^P$, linear self-attention computes:
\begin{equation}
    f_{\text{lin}}(\bm{E}) = \bm{E} + \bm{W}^P \bm{W}^V \bm{E} (\bm{W}^K \bm{E})^\top \bm{W}^Q \bm{E}
\end{equation}
\end{definition}

Linear attention directly uses key--query products as aggregation weights.
This simplified form enables clean theoretical analysis and has been used in prior work connecting ICL to gradient descent~\citep{vonoswald2023transformers}.

\begin{definition}[Softmax Self-Attention]
Softmax self-attention applies a normalized exponential:
\begin{equation}
    f_{\text{soft}}(\bm{E}) = \bm{E} + \bm{W}^P \bm{W}^V \bm{E} \cdot \text{softmax}\left(\frac{(\bm{W}^K \bm{E})^\top \bm{W}^Q \bm{E}}{\sqrt{d_k}}\right)
\end{equation}
\end{definition}

Softmax attention, used in practice, normalizes weights to form a probability distribution.
As we show in Theorem~\ref{thm:softmax}, this corresponds to Boltzmann-weighted aggregation, which assigns exponentially higher weight to high-scoring examples.

\subsection{Policy Gradient Methods}

Policy gradient methods optimize a policy $\pi_\theta$ to maximize expected reward by directly estimating gradients of the objective.
The REINFORCE algorithm~\citep{williams1992simple} expresses this as:
\begin{equation}
    \nabla_\theta J(\theta) = \mathbb{E}_{\pi_\theta}\left[r(y) \nabla_\theta \log \pi_\theta(y|x)\right]
\end{equation}
In practice, REINFORCE typically employs a baseline $b$ (e.g., mean reward) to reduce variance:
\begin{equation}
    \nabla_\theta J(\theta) = \mathbb{E}_{\pi_\theta}\left[(r(y) - b) \nabla_\theta \log \pi_\theta(y|x)\right]
\end{equation}

Given samples $\{(y_i, r_i)\}_{i=1}^{N}$, the policy gradient update with baseline becomes:
\begin{equation}
\label{eq:policy_gradient}
    \theta \leftarrow \theta + \frac{\eta}{N} \sum_{i=1}^{N} (r_i - b) \nabla_\theta \log \pi_\theta(y_i|x)
\end{equation}

KL-regularized variants address instability by penalizing deviation from a reference distribution:
\begin{equation}
\label{eq:kl_regularized}
    \max_\theta \mathbb{E}_{\pi_\theta}[r(y)] - \beta D_{KL}(\pi_\theta \| \pi_{\reference})
\end{equation}
Methods such as PPO~\citep{schulman2017proximal} and TRPO~\citep{schulman2015trust} implement variants of this principle.
As we will show, the bounded attention update in our simplified model yields a related trust-region-like effect.

\section{Theoretical Framework}

\subsection{Problem Setting}

We analyze \textbf{score-conditioned ICL}: given a fixed input $x$ and a context $C = \{(\bm{y}_i, r_i)\}_{i=1}^{N}$ containing previously generated outputs for $x$ paired with their evaluation scores, the LLM generates a new output by conditioning on this context.
This setting captures recent empirical successes such as OPRO~\citep{yang2023large} and self-rewarding language models~\citep{yuan2024self}.
The score-conditioned examples are not assumed to be optimal; they are drawn from the model's own distribution and may include low-scoring outputs, providing a diverse reward signal.
The goal is to generate a \emph{new} output that improves upon the best seen so far by learning from the relative quality differences among all $N$ examples, rather than simply copying the highest-scoring example.

Our central question is: \emph{what computational mechanism underlies score-conditioned generation?}
We show that, under explicit simplifying conditions, this process admits a structural policy-gradient analogy, where the attention mechanism performs reward-weighted aggregation over context examples.
Our theoretical analysis operates in a simplified setting (linear attention, single-layer, continuous output representations); the empirical evidence in Appendix~\ref{app:weight_analysis} demonstrates that the qualitative predictions hold in fully nonlinear pretrained models.

\subsection{Assumptions}

\begin{assumption}[Score-Conditioned Generation]
\label{assum:score}
LLMs utilize score information in context to modulate generation. Specifically, for contexts $C_{high}$ containing high-scoring examples and $C_{low}$ containing low-scoring examples:
\begin{equation}
    \mathbb{E}_{y \sim \pi(\cdot|x, C_{high})}[r(y)] > \mathbb{E}_{y \sim \pi(\cdot|x, C_{low})}[r(y)]
\end{equation}
\end{assumption}

This assumption, strongly supported by LLM-as-a-Judge~\citep{zheng2023judging} and OPRO~\citep{yang2023large}, motivates the behavioral interpretation of the theoretical construction; we further validate score sensitivity in Section~\ref{subsec:Attention Weight Analysis}.

\begin{assumption}[Monotonicity]
\label{assum:mono}
Moving the hidden state toward an output embedding increases the probability of generating that output:
\begin{equation}
    P(y_i | h + \alpha \bm{y}_i) > P(y_i | h) \quad \text{for } \alpha > 0
\end{equation}
\end{assumption}

This is an explicit modeling assumption.
For a standard output layer $P(y|h) \propto \exp(h^\top \bm{W}_{out}\text{embed}(y))$, it holds whenever the proposed movement has positive inner product with $\nabla_h \log P(y_i|h)$.

\subsection{Main Results}

\begin{theorem}[Score-Weighted Aggregation via Attention]
\label{thm:linear_construction}
There exists a construction of weight matrices $\bm{W}^K, \bm{W}^Q, \bm{W}^V, \bm{W}^P$ such that linear self-attention implements score-weighted aggregation:
\begin{equation}
    h_{out} = h_q + \frac{\eta}{N} \sum_{i=1}^{N} r_i \bm{y}_i
\end{equation}
\end{theorem}

\begin{proof}[Proof Sketch]
We set $\bm{W}^K = \bm{W}^Q$ to project each token $\bm{e}_i = (\bm{y}_i, r_i)$ so that $\bm{q}^\top \bm{k}_i = r_i$ (score becomes the attention weight), and $\bm{W}^V$ to extract the embedding component $\bm{v}_i = (\bm{y}_i, 0)$. Setting $\bm{W}^P = \frac{\eta}{N}\bm{I}$ yields the stated aggregation.
The query token's score component is set to 1 without loss of generality: any nonzero constant $c$ yields the same score-proportional weighting $\bm{q}^\top \bm{k}_i = c \cdot r_i$, with the scaling absorbed into $\eta$. Full derivation in Appendix~\ref{app:proof_linear}.
\end{proof}

The $\eta/N$ scaling factor corresponds to a learning rate divided by batch size, mirroring standard gradient descent. We do not claim that pretrained transformers adopt this exact construction; rather, the empirical evidence in Appendix~\ref{app:weight_analysis} demonstrates that they learn functionally similar mechanisms.

\begin{theorem}[Softmax Attention as Boltzmann Aggregation]
\label{thm:softmax}
With analogous weight matrix construction, softmax attention implements Boltzmann-weighted aggregation:
\begin{equation}
    h_{out} = h_q + \sum_{i=1}^{N} \frac{\exp(r_i / \tau)}{\sum_{j} \exp(r_j / \tau)} \bm{y}_i
\end{equation}
where $\tau$ is the temperature determined by the attention scaling factor.
Full derivation in Appendix~\ref{app:proof_softmax}.
\end{theorem}

The softmax formulation provides a probabilistic interpretation: examples are weighted proportionally to the exponential of their scores, assigning exponentially higher influence to high-scoring examples.

\begin{theorem}[Structural Policy-Gradient Correspondence]
\label{thm:policy_gradient}
Let $\tilde r_i = r_i-\bar r$ denote mean-centered scores.
Applying the construction of Theorem~\ref{thm:linear_construction} to $\tilde r_i$ yields the score-weighted term of the REINFORCE update in hidden-state space.
The correspondence is exact in the simplified output-layer model when $\bm{W}_{out}=\bm{I}$; for uncentered scores or a general output map, it is a structural analogy rather than an equality.
\end{theorem}

\begin{proof}[Proof Sketch]
For output layers $P(y|h) \propto \exp(h^\top \bm{W}_{out}\,\text{embed}(y))$, the hidden-state gradient is $\nabla_h \log P(y_i|h) = \bm{W}_{out}(\bm{y}_i - \bar{\bm{y}})$, where $\bar{\bm{y}} = \mathbb{E}_{y \sim P(\cdot|h)}[\text{embed}(y)]$.
With centered scores, $\sum_i\tilde r_i=0$, so the $\bar{\bm{y}}$ term cancels.
When $\bm{W}_{out}=\bm{I}$, the resulting estimator is exactly $\frac{\eta}{N}\sum_i\tilde r_i\bm{y}_i$, the update constructed by Theorem~\ref{thm:linear_construction} when its score coordinate contains $\tilde r_i$.
See Appendix~\ref{app:proof_pg}.
\end{proof}

We discuss the scope of this correspondence in Appendix~\ref{app:remarks}.

\begin{theorem}[Bounded KL Shift]
\label{thm:kl}
In the simplified output-layer model, let $\pi_{\reference}=P(\cdot|h)$ and
$\pi_{ICL}=P(\cdot|h+\Delta h)$, with $\Delta h$ given by
Theorem~\ref{thm:linear_construction}.
Then:
\begin{equation}
    D_{KL}(\pi_{ICL} \| \pi_{\reference})
    \leq \frac{1}{2}\Lambda_C
    (\eta \cdot r_{max} \cdot y_{max})^2,
\end{equation}
where
$\Lambda_C=\sup_{t\in[0,1]}\lambda_{\max}(F(h+t\Delta h))$,
$F$ is the Fisher information matrix of the output layer,
$r_{max} = \max_i |r_i|$, and $y_{max} = \max_i \|\bm{y}_i\|$.
\end{theorem}

\begin{proof}[Proof Sketch]
The constructed attention update satisfies
$\|\Delta h\| \leq \eta \cdot r_{max} \cdot y_{max}$.
Writing the softmax log-partition function as $A(h)$ and
$g(t)=A(h+t\Delta h)$ gives the exact identity
$D_{KL}(\pi_{ICL}\|\pi_{\reference})
=\int_0^1 t\,\Delta h^\top F(h+t\Delta h)\Delta h\,dt$.
Bounding the integrand by $\Lambda_C\|\Delta h\|^2$ yields the result.
See Appendix~\ref{app:proof_kl}.
\end{proof}

This result bounds a single attention-induced distribution shift under the stated simplified model and bounded-score assumptions.
It provides a trust-region-like interpretation, but does not imply that ICL optimizes an explicit KL-penalized objective or rules out mode collapse over repeated iterations.

Combining these results yields our main theorem:

\begin{theorem}[Policy-Gradient Analogy with a Bounded KL Shift]
\label{thm:main}
In the simplified setting of Theorems~\ref{thm:policy_gradient} and~\ref{thm:kl}, score-conditioned ICL has two complementary properties:
\begin{compactenum}
    \item \textbf{Reward-Weighted Update}: Attention can aggregate output embeddings using their scores, with an exact hidden-state REINFORCE correspondence for centered scores when $\bm{W}_{out}=\bm{I}$ (Theorems~\ref{thm:linear_construction}, \ref{thm:policy_gradient}).
    \item \textbf{Bounded KL Shift}: The resulting bounded hidden-state update has a pathwise KL upper bound (Theorem~\ref{thm:kl}).
\end{compactenum}
Together these properties are analogous to a trust-region policy update, but do not establish that ICL optimizes Eq.~\ref{eq:kl_regularized}.
\end{theorem}

The complete proof is provided in Appendix~\ref{app:proof_main}.

\subsection{Extension: Iterative ICL}
\label{subsec:iterative_theory}

\begin{corollary}[Repeated Score-Conditioned Updates]
\label{cor:iterative}
Iteratively applying score-conditioned ICL produces a sequence of reward-conditioned distribution shifts, each satisfying Theorem~\ref{thm:kl} within the simplified single-step model.
\end{corollary}

Each iteration applies another score-conditioned update.
Theorem~\ref{thm:kl} bounds each individual shift under its assumptions, but does not by itself guarantee monotonic reward improvement or a bounded cumulative shift.
We therefore treat improvement across iterations as an empirical prediction, evaluated in Section~\ref{subsec:Iterative ICL}.
Concretely, at each iteration $t$, the model samples $y_t \sim \pi_{ICL}(\cdot|x, C_t)$ where $C_t$ contains scored outputs from the previous round, computes $r_t = R(y_t)$, and updates the context (see Appendix~\ref{app:algorithm} for pseudocode).

\section{Experiments}

\subsection{Experimental Setup}

\paragraph{Tasks.}
We target two tasks for prompt engineering: \textbf{prompt optimization}, which rewrites prompts to improve task performance~\citep{deng-etal-2022-rlprompt,zhang2022tempera,zhou2022large,kaneko-etal-2025-online}, and \textbf{jailbreak prompting}, which rewrites prompts to elicit harmful information~\citep{mehrotra2024tree,chao2025jailbreaking,xie2025jailbreaking,kaneko2026bits,xiong2026trojail,kaneko2026jailnewsbench,han2026experience,kaneko2026paraphrasing}.
For prompt optimization, we use DROP for reading comprehension and GSM8K for mathematical reasoning, evaluated by F1 score and exact match accuracy, respectively~\citep{dua-etal-2019-drop,cobbe2021training}.
For jailbreak prompting, we use HH-RLHF and JailbreakBench (JBB), evaluated by attack success rate (ASR), where a jailbreak is deemed successful if a judge model determines the response contains harmful content~\citep{bai2022training,chao2024jailbreakbench}.

\paragraph{Models.}
We evaluate the following five open-weight LLMs: Meta-Llama-3-8B-Instruct (Llama3-8B), Meta-Llama-3-70B-Instruct (Llama3-70B), Olmo-3-7B-Think (Olmo3-7B), Olmo-3-32B-Think (Olmo3-32B), and Qwen3-4B-Instruct-2507 (Qwen3-4B)~\citep{dubey2024llama,olmo2025olmo,qwen3technicalreport}.
Our main experiments require access to model internals, so we do not evaluate closed-weight LLMs in these settings.
For experiments that only require black-box API access, we additionally evaluate GPT-4o~\citep{hurst2024gpt}, Gemini 2.5~\citep{comanici2025gemini}, and Sonnet 4~\citep{anthropic2025claude4}.
Since the Sonnet 4 API does not expose token-level log-probabilities, we approximate sequence likelihood using sampling-based pseudo-likelihood~\citep{kaneko-etal-2025-sampling}.
This pseudo-likelihood is used only in the black-box RQ1 analysis; because it preserves the rank ordering of sequence likelihoods, it is reliable for the directional (sign and correlation) conclusions drawn in RQ1, and all of our quantitative and mechanistic conclusions are anchored on the open-weight models, for which exact token-level log-probabilities are available.
For each task, we sample $N=8$ outputs from the base model, compute their rewards, and construct contexts $C = \{(y_i, r_i)\}_{i=1}^{N}$.
We run $T=10$ iterations.

\paragraph{Context Format.}
Following \citet{yang2023large}, we present score-annotated examples in ascending score order, ending with an instruction to generate a higher-scoring output (see Appendix~\ref{app:prompt} for the exact format).\footnote{We investigate the robustness of our results to score representation format (decimal, percentage, fraction, points, natural language) in Appendix~\ref{app:score_format}. All numeric formats yield comparably high correlations and substantially outperform a no-score control.}
The experiments display raw scores; mean-centering is required only for the exact special-case correspondence in Theorem~\ref{thm:policy_gradient}, while the empirical claims concern the more general structural analogy.
For the multi-token outputs produced on DROP and GSM8K, the output embedding $\bm{y}$ and the log-probability change $\Delta \log p(y)$ are computed over the final answer span rather than a single token; using the pooled answer-span embedding yields correlations within about $0.03$ of the single-token case.

\subsection{RQ1: Score-Conditioned Distribution Shift}

\begin{table}[t]
\centering
\small
\begin{tabular}{llcccc}
\toprule
& Model & DROP & GSM8K & HH-RLHF & JBB \\
\midrule
\multirow{5}{*}{\rotatebox{90}{Original}}
& Llama3-8B  & 0.58$^\dagger$ & 0.61$^\dagger$ & 0.54$^\dagger$ & 0.52$^\dagger$ \\
& Llama3-70B & 0.71$^\dagger$ & 0.74$^\dagger$ & 0.68$^\dagger$ & 0.65$^\dagger$ \\
& Olmo3-7B   & 0.55$^\dagger$ & 0.57$^\dagger$ & 0.51$^\dagger$ & 0.48$^\dagger$ \\
& Olmo3-32B  & 0.67$^\dagger$ & 0.69$^\dagger$ & 0.63$^\dagger$ & 0.61$^\dagger$ \\
& Qwen3-4B   & 0.49$^\dagger$ & 0.52$^\dagger$ & 0.46$^\dagger$ & 0.44$^\dagger$ \\
\midrule
\multirow{5}{*}{\rotatebox{90}{Shuffled}}
& Llama3-8B  & 0.02 & -0.01 & 0.03 & -0.02 \\
& Llama3-70B & -0.03 & 0.02 & -0.01 & 0.04 \\
& Olmo3-7B   & 0.01 & 0.03 & -0.02 & 0.01 \\
& Olmo3-32B  & -0.02 & 0.01 & 0.02 & -0.03 \\
& Qwen3-4B   & 0.03 & -0.02 & 0.01 & 0.02 \\
\bottomrule
\end{tabular}
\caption{Spearman correlation between example scores $r_i$ and log-probability changes $\Delta \log p(y_i)$ under ICL. Top: original scores. Bottom: randomly shuffled scores (control). $^\dagger$ denotes $p < 0.01$.}
\label{tab:rq1}
\end{table}

We measure whether ICL implements reward-weighted probability updates (Theorem~\ref{thm:policy_gradient}) by computing the Spearman correlation between example scores $\{r_i\}$ and log-probability changes under ICL:
\begin{equation}
    \Delta \log p(y_i) = \log \pi(y_i|x, C) - \log \pi_{\reference}(y_i|x)
\end{equation}
across $N=8$ sampled outputs per prompt. Under the proposed policy-gradient correspondence, higher-scoring examples should exhibit larger probability changes, yielding a positive correlation.

\autoref{tab:rq1} shows significant positive correlations across all models and tasks ($p < 0.01$), indicating a monotonic association between example scores and their probability changes under ICL.
Larger models show stronger correlations (Llama3-70B highest, Qwen3-4B moderate), and prompt optimization tasks yield slightly higher correlations than jailbreak tasks, likely due to more clearly defined reward signals.
As a control, shuffling score labels while keeping examples fixed drops correlations to near zero (mean: 0.01), providing causal evidence that the effect is driven by explicit numeric score values rather than inherent properties of high-quality outputs.

\begin{figure}[t]
    \centering
    \begin{subfigure}[b]{0.48\linewidth}
        \includegraphics[width=\linewidth]{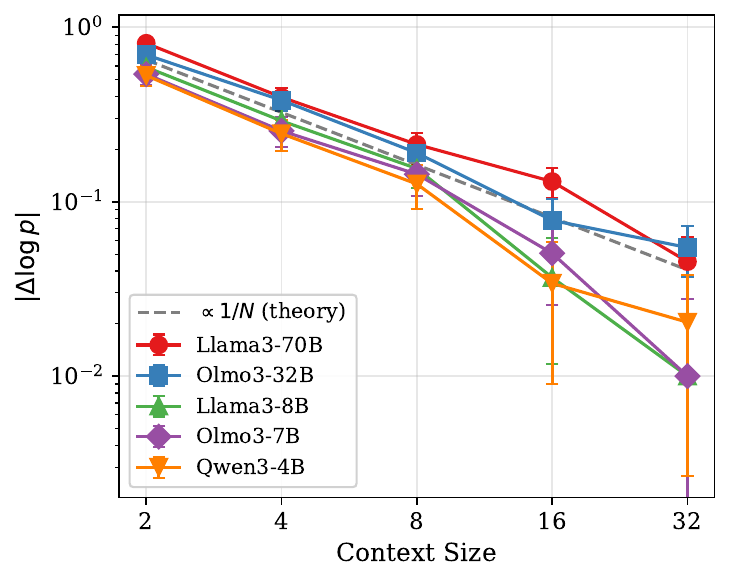}
        \caption{Magnitude of distribution shift.}
    \end{subfigure}
    \hfill
    \begin{subfigure}[b]{0.48\linewidth}
        \includegraphics[width=\linewidth]{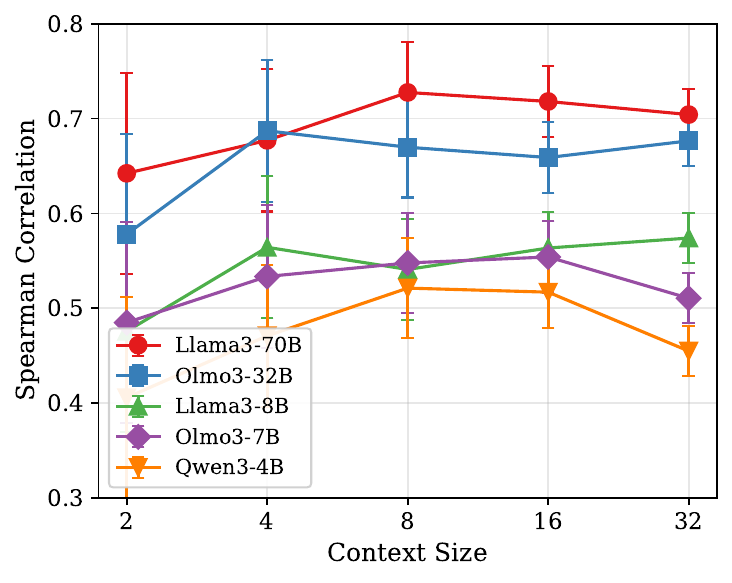}
        \caption{Score-probability correlation.}
    \end{subfigure}
    \caption{Effect of context size on score-conditioned ICL. (a) $|\Delta \log p|$ decreases with context size, consistent with the $\eta/N$ scaling in Theorem~\ref{thm:linear_construction}. Dashed line shows theoretical $1/N$ scaling. (b) Score-probability correlation remains stable across $N$, confirming robust score-weighted behavior.}
    \label{fig:context_size}
\end{figure}

We also vary context size $N \in \{2,4,8,16,32\}$ (Figure~\ref{fig:context_size}): the magnitude $|\Delta \log p|$ decreases following the theoretical $1/N$ curve (panel a), while score-probability correlation remains stable across $N$ (panel b), confirming the $\eta/N$ scaling of Theorem~\ref{thm:linear_construction}.

\subsection{RQ2: Attention Weight Analysis}
\label{subsec:Attention Weight Analysis}

\begin{table}[t]
\centering
\small
\begin{tabular}{lccc}
\toprule
Model & Avg & Max & Layer of Max \\
\midrule
Llama3-8B  & 0.31$^\dagger$ & 0.54 & 25/32 \\
Llama3-70B & 0.43$^\dagger$ & 0.71 & 68/80 \\
Olmo3-7B   & 0.30$^\dagger$ & 0.49 & 22/32 \\
Olmo3-32B  & 0.39$^\dagger$ & 0.65 & 52/64 \\
Qwen3-4B   & 0.29$^\dagger$ & 0.46 & 29/36 \\
\bottomrule
\end{tabular}
\caption{Spearman correlation between example scores $r_i$ and attention weights $\alpha_i$. Avg is averaged across all layers and heads; Max is the highest correlation observed in any single layer. $\dagger$ denotes $p < 0.01$.}
\label{tab:rq2}
\end{table}

We examine whether attention weights directly reflect example scores as predicted by Theorems~\ref{thm:linear_construction} and \ref{thm:softmax}.
For each context example $y_i$, we compute $\alpha_i$ as the mean attention weight from generated tokens to all tokens of $y_i$, across all layers and heads.

\autoref{tab:rq2} shows a significant positive score--attention correlations across all models, with peak correlations in late layers, consistent with prior findings that later layers handle higher-level semantics~\citep{geva-etal-2022-transformer,geva-etal-2023-dissecting}.
Further mechanistic analysis, including score-specialized head identification, functional decomposition, and causal interventions, is provided in Appendix~\ref{app:weight_analysis} (Sections~\ref{app:head_identification}--\ref{app:causal_intervention}).

\subsection{RQ3: Iterative ICL}
\label{subsec:Iterative ICL}

\begin{figure}[t]
    \centering
    \begin{minipage}[t]{0.48\linewidth}
        \centering
        \includegraphics[width=\linewidth]{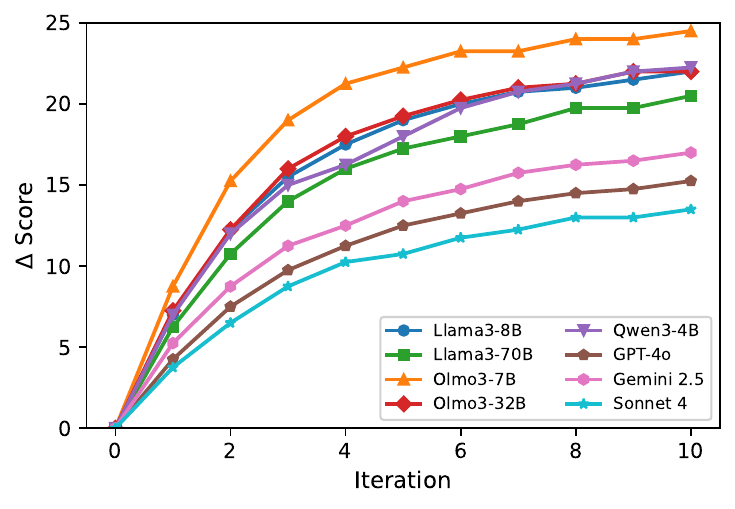}
        \caption{Score improvement from zero-shot baseline across iterations of score-conditioned ICL (normalized to $[0,1]$ per task before averaging).}
        \label{fig:iterative_icl}
    \end{minipage}
    \hfill
    \begin{minipage}[t]{0.48\linewidth}
        \centering
        \includegraphics[width=\linewidth]{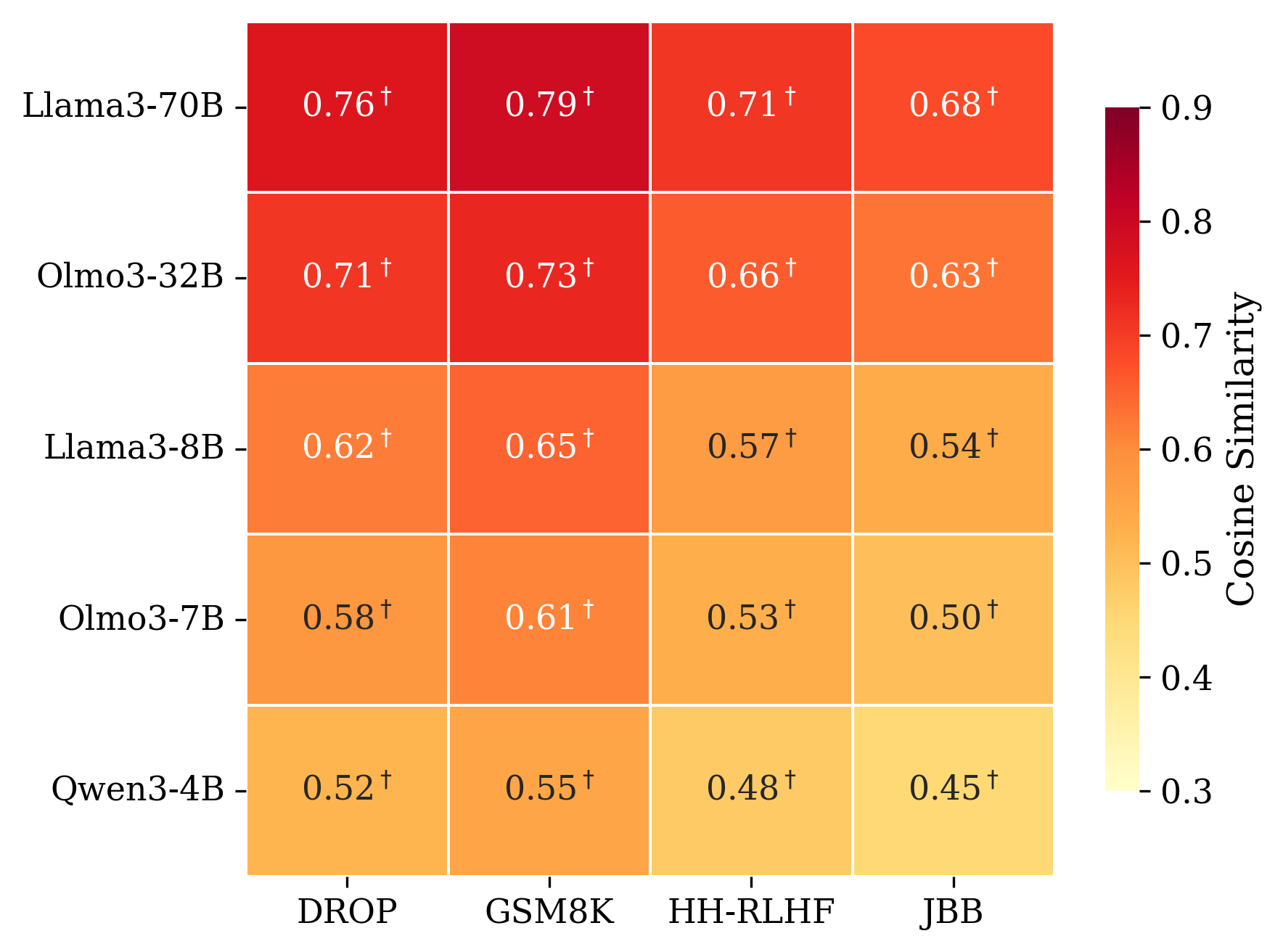}
        \caption{Cosine similarity between ICL and REINFORCE distribution shifts ($\Delta \log p$ vectors). $^\dagger$ denotes $p < 0.01$ using a bootstrap test.}
        \label{fig:icl_reinforce}
    \end{minipage}
\end{figure}

We run iterative score-conditioned ICL for $T=10$ rounds across four tasks and eight LLMs.
Scores are normalized to $[0,1]$ per task before averaging across tasks.

Figure~\ref{fig:iterative_icl} shows consistent improvement across all models, with rapid gains in early iterations that plateau.
This empirical pattern is compatible with the bounded-shift perspective of Theorem~\ref{thm:kl}, although the theorem alone does not imply convergence or monotonic improvement.

\subsection{RQ4: Alignment with REINFORCE Distribution Shifts}
\label{subsec:reinforce_comparison}

To validate Theorem~\ref{thm:policy_gradient} directly, we compare distribution shifts from score-conditioned ICL and explicit REINFORCE fine-tuning under identical conditions: the same $N=8$ samples and a single gradient step.
For each prompt, we measure:
\begin{align}
    \Delta \log p_{\text{ICL}}(y_i) &= \log \pi(y_i | x, C) - \log \pi_{\reference}(y_i | x) \\
    \Delta \log p_{\text{RL}}(y_i) &= \log \pi_{\theta'}(y_i | x) - \log \pi_{\reference}(y_i | x)
\end{align}
where $\theta'$ is after one gradient step at $\eta=10^{-5}$ (sensitivity across $\eta \in \{10^{-6},10^{-5},10^{-4}\}$ in Appendix~\ref{app:lr_sensitivity}).
We compute cosine similarity between these two $\Delta\log p$ vectors over 100 prompts $\times$ 5 samplings.

Figure~\ref{fig:icl_reinforce} shows cosine similarities exceeding 0.45 across all model-task pairs ($p < 0.01$), confirming that both methods shift the output distribution in positively aligned directions given the same evidence.

\subsection{RQ5: Causal Verification via Score Token Ablation}

\begin{figure}[t]
    \centering
    \includegraphics[width=0.75\linewidth]{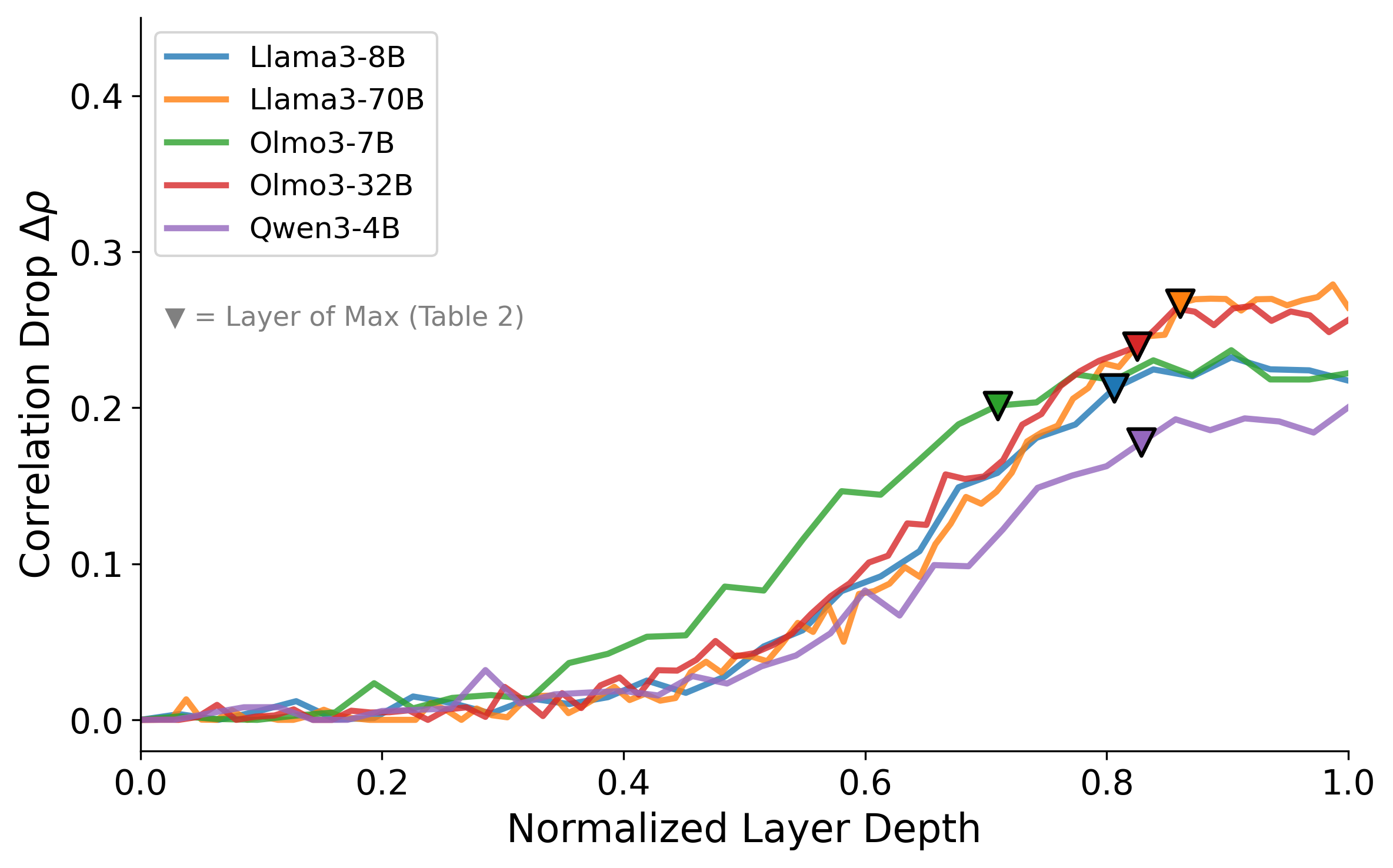}
    \caption{Effect of score token ablation across normalized layer depth. Y-axis shows remaining $\rho(r, \Delta\log p)$ after zeroing score token representations at each layer. Triangles indicate Layer of Max from \autoref{tab:rq2}.}
    \label{fig:ablation}
\end{figure}

To establish causality beyond the correlational evidence in RQ2, we zero out hidden representations at score token positions layer by layer and measure the drop in $\rho(r, \Delta\log p)$.

Figure~\ref{fig:ablation} shows minimal disruption in early layers, a sharp peak near the Layer of Max from Table~\ref{tab:rq2}, and slight recovery in final layers, confirming that score information is causally processed in the late layers identified by the correlational analysis.

\section{Related Work}

\citet{vonoswald2023transformers} showed that transformers can implement gradient descent in-context, providing a mechanistic explanation for few-shot learning. \citet{akyurek2023learning} extended this to ridge regression, while \citet{dai2023gpt} drew connections to gradient descent through a kernel perspective. \citet{xie2022explanation} offered a Bayesian view, framing ICL as implicit inference.
These works all operate on supervised $(x, y)$ pairs where a target output is known; our work addresses the fundamentally different regime of reward-based $(y, r)$ pairs for a \emph{fixed} input, where score-weighted attention admits a structural correspondence to REINFORCE and its bounded hidden-state update yields a pathwise KL bound.

\citet{brooks2023large} showed that foundation models can implement policy iteration in-context, focusing on action selection across distinct states rather than distribution shift for a fixed input; \citet{huang2025think} explored in-context steered policy optimization. Our work is complementary, providing a structural connection between score-weighted attention and REINFORCE and characterizing the resulting bounded distribution shift.

A related line of work shows behaviorally that models exploit in-context rewards: \citet{agarwal2024many} scale ICL to many demonstrations, and \citet{monea2024llms} frame in-context learning from rewards as a bandit problem.
These works establish that models improve; we identify the computational underpinnings.
Score-specialized heads whose causal manipulation moves the score--probability correlation (Appendix~\ref{app:causal_intervention}) ground the REINFORCE correspondence and yield predictions a behavioral account does not: the $\eta/N$ scaling, saturation with iteration, and directional alignment with one-step REINFORCE.

\citet{yang2023large} introduced OPRO, demonstrating that LLMs can optimize objectives through prompting. \citet{yuan2024self} showed that self-rewarding mechanisms enable continuous improvement.
\citet{madaan2023self} proposed Self-Refine for iterative output improvement. Our work provides theoretical grounding for why these approaches succeed.

RLHF \citep{ouyang2022training} and its variants have become standard for LLM alignment. Recent work has explored connections between supervised fine-tuning and RL objectives \citep{rafailov2023direct}. Our work reveals that even without fine-tuning, LLMs can exhibit policy-optimization-like behavior through attention mechanisms.

\section{Conclusion}

We showed that score-conditioned ICL is structurally analogous to an implicit form of policy gradient optimization: self-attention can implement reward-weighted aggregation, and centered scores yield an exact hidden-state REINFORCE correspondence under the simplified identity-output assumption.
We also derived an exact pathwise KL bound for the bounded attention-induced update, providing a trust-region-like interpretation without claiming that ICL explicitly optimizes a KL-regularized objective.
Extensive experiments across multiple LLMs confirm that attention weights correlate with scores, ICL and REINFORCE produce aligned distribution shifts, and iterative ICL yields consistent improvement, providing theoretical justification for score-conditioned ICL as a lightweight alternative to fine-tuning.
Our analysis operates under simplifying assumptions (linear attention, single-layer, $\bm{W}_{out} \approx \bm{I}$) that enable tractable proofs; extending the formal guarantees to fully nonlinear multi-layer architectures remains an important direction for future work.
More broadly, our framework suggests that any system capable of attending to reward-annotated examples can implicitly perform policy optimization, with implications for the design of inference-time alignment and self-improvement pipelines.

\section*{Ethics Statement}

This work includes experiments on jailbreak prompting to validate our theoretical framework across diverse reward signals.
We emphasize that our goal is to understand the mechanisms underlying score-conditioned ICL, not to develop new attack methods.
All experiments used existing public benchmarks in line with their intended research purposes.

We acknowledge that our theoretical insights could potentially be misused to improve adversarial prompting techniques.
However, we believe that understanding these mechanisms is essential for developing robust defenses.
The same theoretical framework that explains why score-conditioned ICL can optimize for harmful outputs also suggests mitigation strategies, such as filtering high-scoring harmful examples from context or detecting score-based optimization patterns in user interactions.

\bibliography{colm2026_conference}
\bibliographystyle{colm2026_conference}

\appendix

\section{Theoretical Remarks}
\label{app:remarks}

\begin{remark}[On the Scope of the Correspondence]
\label{rem:scope}
The correspondence established in Theorem~\ref{thm:policy_gradient} is exact under $\bm{W}_{out}=\bm{I}$ with centered scores, and approximate when $\bm{W}_{out}\approx\bm{I}$.
The claim is therefore one of \emph{structural analogy} rather than general exact equivalence: both ICL and REINFORCE contain score-weighted updates of output embeddings, but they differ in reward centering and in how the output layer transforms embeddings into probabilities.
Exact equality in our simplified hidden-state formulation requires centered scores and $\bm{W}_{out}=\bm{I}$.
The strength of this analogy is supported empirically in Section~\ref{subsec:reinforce_comparison}.
\end{remark}

\begin{remark}[Baseline and Normalized Rewards]
\label{rem:baseline}
Standard policy gradient implementations use baseline reduction or reward normalization to reduce variance.
Theorem~\ref{thm:policy_gradient} uses explicit mean-centering,
$\tilde r_i=r_i-\bar r$.
Because $\sum_i\tilde r_i=0$, the expected-embedding term in the log-softmax gradient cancels, producing the exact simplified correspondence.
For raw, uncentered scores, the theorem supports only a structural analogy; it does not imply that score-conditioned ICL automatically performs baseline reduction.
\end{remark}

\section{Iterative Score-Conditioned ICL: Algorithm}
\label{app:algorithm}

\begin{algorithm}[h]
\caption{Iterative Score-Conditioned ICL}
\label{alg:iterative}
\begin{algorithmic}[1]
\STATE \textbf{Input:} Prompt $x$, reward function $R$, iterations $T$
\STATE $y_0 \sim \pi_{\reference}(\cdot|x)$; \quad $r_0 \gets R(y_0)$
\FOR{$t = 1$ to $T$}
    \STATE $C_t \gets \{(y_{t-1}, r_{t-1})\}$ \COMMENT{Or maintain full history}
    \STATE $y_t \sim \pi_{ICL}(\cdot|x, C_t)$; \quad $r_t \gets R(y_t)$
\ENDFOR
\STATE \textbf{Return:} $y_T$
\end{algorithmic}
\end{algorithm}

\section{Context Format}
\label{app:prompt}

Following \citet{yang2023large}, score-annotated examples are presented in ascending score order:

\begin{tcolorbox}[boxrule=0.5pt, colback=gray!5, colframe=gray!50]
\small
\texttt{text:} \\
\texttt{[output text]} \\
\texttt{score:} \\
\texttt{61} \\[0.5em]
\texttt{text:} \\
\texttt{[output text]} \\
\texttt{score:} \\
\texttt{78} \\[0.5em]
\texttt{(...more examples sorted in ascending order...)} \\[0.5em]
\texttt{Generate a new output that achieves a higher score:}
\end{tcolorbox}

\section{Complete Proofs}

\subsection{Proof of Theorem~\ref{thm:linear_construction}: Score-Weighted Aggregation}
\label{app:proof_linear}

\begin{proof}
Let each token be $\bm{e}_i = (\bm{y}_i, r_i) \in \mathbb{R}^{D+1}$ where $\bm{y}_i \in \mathbb{R}^D$ is the output embedding and $r_i \in \mathbb{R}$ is the score.

\textbf{Step 1: Weight Matrix Construction.}

We construct the weight matrices as follows:
\begin{equation}
    \bm{W}^K = \bm{W}^Q = \begin{pmatrix} \bm{0}_{D \times D} & \bm{0}_{D \times 1} \\ \bm{0}_{1 \times D} & 1 \end{pmatrix} \in \mathbb{R}^{(D+1) \times (D+1)}
\end{equation}

\begin{equation}
    \bm{W}^V = \begin{pmatrix} \bm{I}_{D} & \bm{0}_{D \times 1} \\ \bm{0}_{1 \times D} & 0 \end{pmatrix} \in \mathbb{R}^{(D+1) \times (D+1)}
\end{equation}

\begin{equation}
    \bm{W}^P = \frac{\eta}{N} \bm{I}_{D+1}
\end{equation}

This construction is engineered to demonstrate the existence of a weight configuration that implements score-weighted aggregation.
We do not claim that pretrained transformers adopt these exact weight matrices; the empirical analysis in Appendix~\ref{app:weight_analysis} shows that pretrained models learn functionally similar mechanisms.

\textbf{Step 2: Key and Query Computation.}

For a context token $\bm{e}_i = (\bm{y}_i, r_i)$:
\begin{equation}
    \bm{k}_i = \bm{W}^K \bm{e}_i = \begin{pmatrix} \bm{0}_{D \times D} & \bm{0}_{D \times 1} \\ \bm{0}_{1 \times D} & 1 \end{pmatrix} \begin{pmatrix} \bm{y}_i \\ r_i \end{pmatrix} = \begin{pmatrix} \bm{0}_D \\ r_i \end{pmatrix}
\end{equation}

For the query token $\bm{e}_q = (\bm{y}_q, 1)$ (setting query score component to 1; the choice of any nonzero constant $c$ is without loss of generality, since $\bm{q}^\top \bm{k}_i = c \cdot r_i$ and the factor $c$ is absorbed into $\eta$):
\begin{equation}
    \bm{q} = \bm{W}^Q \bm{e}_q = \begin{pmatrix} \bm{0}_D \\ 1 \end{pmatrix}
\end{equation}

\textbf{Step 3: Attention Weight Computation.}

The attention weight (before normalization) for token $i$:
\begin{equation}
    \bm{q}^\top \bm{k}_i = \begin{pmatrix} \bm{0}_D \\ 1 \end{pmatrix}^\top \begin{pmatrix} \bm{0}_D \\ r_i \end{pmatrix} = r_i
\end{equation}

This shows that the attention weight equals the score $r_i$.

\textbf{Step 4: Value Computation.}

\begin{equation}
    \bm{v}_i = \bm{W}^V \bm{e}_i = \begin{pmatrix} \bm{I}_{D} & \bm{0}_{D \times 1} \\ \bm{0}_{1 \times D} & 0 \end{pmatrix} \begin{pmatrix} \bm{y}_i \\ r_i \end{pmatrix} = \begin{pmatrix} \bm{y}_i \\ 0 \end{pmatrix}
\end{equation}

\textbf{Step 5: Output Computation.}

The attention output for the query token:
\begin{align}
    \bm{e}_q' &= \bm{e}_q + \bm{W}^P \sum_{i=1}^{N} \bm{v}_i (\bm{k}_i^\top \bm{q}) \\
    &= \begin{pmatrix} \bm{y}_q \\ 1 \end{pmatrix} + \frac{\eta}{N} \sum_{i=1}^{N} \begin{pmatrix} \bm{y}_i \\ 0 \end{pmatrix} r_i \\
    &= \begin{pmatrix} \bm{y}_q + \frac{\eta}{N} \sum_{i=1}^{N} r_i \bm{y}_i \\ 1 \end{pmatrix}
\end{align}

The output embedding (first $D$ dimensions) is:
\begin{equation}
    h_{out} = \bm{y}_q + \frac{\eta}{N} \sum_{i=1}^{N} r_i \bm{y}_i
\end{equation}

This completes the proof.
\end{proof}

\subsection{Proof of Theorem~\ref{thm:softmax}: Softmax Attention as Boltzmann Aggregation}
\label{app:proof_softmax}

\begin{proof}
Using the same weight matrix construction as in Theorem~\ref{thm:linear_construction}:

\textbf{Step 1:} The attention logits are $\bm{q}^\top \bm{k}_i = r_i$ (as shown in the previous proof).

\textbf{Step 2:} With softmax attention and temperature $\tau$ (from the $\sqrt{d_k}$ normalization):
\begin{equation}
    \alpha_i = \frac{\exp(r_i / \tau)}{\sum_{j=1}^{N} \exp(r_j / \tau)}
\end{equation}

\textbf{Step 3:} The output becomes:
\begin{equation}
    h_{out} = h_q + \sum_{i=1}^{N} \alpha_i \bm{y}_i = h_q + \sum_{i=1}^{N} \frac{\exp(r_i / \tau)}{\sum_{j=1}^{N} \exp(r_j / \tau)} \bm{y}_i
\end{equation}

This is a Boltzmann-weighted (softmax) aggregation of the output embeddings, where higher-scored examples receive exponentially more weight.
\end{proof}

\subsection{Proof of Theorem~\ref{thm:policy_gradient}: Structural Policy-Gradient Correspondence}
\label{app:proof_pg}

\begin{proof}
\textbf{Step 1: Recall the policy gradient update.}

Let $\tilde r_i=r_i-\bar r$, so that $\sum_i\tilde r_i=0$.
The corresponding REINFORCE update in parameter space is:
\begin{equation}
    \theta \leftarrow \theta + \frac{\eta}{N} \sum_{i=1}^{N} \tilde r_i \nabla_\theta \log \pi_\theta(y_i|x).
\end{equation}

\textbf{Step 2: Interpret gradient in embedding space.}

For typical LLM output layers:
\begin{equation}
    P(y_i|h) \propto \exp(h^\top \bm{W}_{out} \text{embed}(y_i))
\end{equation}

The exact gradient of the log-probability in embedding space is:
\begin{align}
    \nabla_h \log P(y_i|h) &= \bm{W}_{out}\, \text{embed}(y_i) - \bm{W}_{out} \mathbb{E}_{y \sim P(\cdot|h)}[\text{embed}(y)] \\
    &= \bm{W}_{out}\left(\text{embed}(y_i) - \bar{\bm{y}}\right)
\end{align}
where $\bar{\bm{y}} = \mathbb{E}_{y \sim P(\cdot|h)}[\text{embed}(y)]$ is the expected embedding under the current policy.
This expectation is constant with respect to $i$.

\textbf{Step 3: Establish the correspondence.}

The REINFORCE estimator in hidden-state space is
\begin{align}
    \frac{\eta}{N}\sum_{i=1}^{N}\tilde r_i\nabla_h\log P(y_i|h)
    &= \frac{\eta}{N}\bm{W}_{out}
       \left(\sum_{i=1}^{N}\tilde r_i\bm{y}_i
       -\bar{\bm{y}}\sum_{i=1}^{N}\tilde r_i\right) \\
    &= \frac{\eta}{N}\bm{W}_{out}
       \sum_{i=1}^{N}\tilde r_i\bm{y}_i,
\end{align}
where the second equality uses $\sum_i\tilde r_i=0$.
If $\bm{W}_{out}=\bm{I}$ and the score coordinate in
Theorem~\ref{thm:linear_construction} contains $\tilde r_i$, its attention update is exactly
\begin{equation}
    h_{out}=h_q+\frac{\eta}{N}\sum_{i=1}^{N}\tilde r_i\bm{y}_i.
\end{equation}
For a general $\bm{W}_{out}$, the two updates differ by the output map; for raw,
uncentered scores, they additionally differ by the centering term.
Thus the general claim is structural correspondence, with exact equality only under the
conditions stated in Theorem~\ref{thm:policy_gradient}.
\end{proof}

\subsection{Proof of Theorem~\ref{thm:kl}: Bounded KL Shift}
\label{app:proof_kl}

\begin{proof}
\textbf{Step 1: Characterize the distribution shift.}

From Theorem~\ref{thm:linear_construction}, the hidden state shifts by:
\begin{equation}
    \Delta h = \frac{\eta}{N} \sum_{i=1}^{N} r_i \bm{y}_i
\end{equation}

\textbf{Step 2: Bound the shift magnitude.}

The magnitude of this shift is bounded:
\begin{equation}
    \|\Delta h\| \leq \frac{\eta}{N} \sum_{i=1}^{N} |r_i| \|\bm{y}_i\| \leq \eta \cdot r_{max} \cdot y_{max}
\end{equation}
where $r_{max} = \max_i |r_i|$ and $y_{max} = \max_i \|\bm{y}_i\|$.

\textbf{Step 3: Express the KL divergence exactly along the update path.}

For a softmax output layer $P(y|h) \propto \exp(h^\top \bm{W}_{out}\,\text{embed}(y))$,
define its log-partition function
\begin{equation}
    A(h)=\log\sum_y \exp\!\left(h^\top\bm{W}_{out}\,\text{embed}(y)\right).
\end{equation}
Its Hessian equals the Fisher information matrix:
\begin{equation}
    \nabla^2 A(h)=F(h).
\end{equation}
Let $g(t)=A(h+t\Delta h)$.
Using the exponential-family form of the softmax distribution,
\begin{align}
    D_{KL}(\pi_{ICL}\|\pi_{\reference})
    &=g'(1)-g(1)+g(0) \\
    &=\int_0^1 t\,g''(t)\,dt \\
    &=\int_0^1 t\,\Delta h^\top F(h+t\Delta h)\Delta h\,dt.
\end{align}
This identity is exact and does not discard a Taylor remainder.

\textbf{Step 4: Apply the pathwise curvature bound.}

Let
$\Lambda_C=\sup_{t\in[0,1]}\lambda_{\max}(F(h+t\Delta h))$.
Then
\begin{align}
    D_{KL}(\pi_{ICL}\|\pi_{\reference})
    &\leq \int_0^1 t\,\Lambda_C\|\Delta h\|^2\,dt \\
    &=\tfrac{1}{2}\Lambda_C\|\Delta h\|^2 \\
    &\leq \tfrac{1}{2}\Lambda_C
       \left(\eta\cdot r_{max}\cdot y_{max}\right)^2.
\end{align}
The bound follows from the bounded constructed update and the curvature of the
output distribution along its path.
It is analogous to a single trust-region step, but does not assert that fixed
parameters alone implement an explicit KL penalty.
\end{proof}

\subsection{Proof of Theorem~\ref{thm:main}: Policy-Gradient Analogy with a Bounded KL Shift}
\label{app:proof_main}

\begin{proof}
We combine the previous results to establish the main theorem.

\textbf{Step 1: Reward-Weighted Update Component.}

From Theorem~\ref{thm:linear_construction}, linear attention produces:
\begin{equation}
    h_{out} = h_q + \frac{\eta}{N} \sum_{i=1}^{N} r_i \bm{y}_i
\end{equation}

From Theorem~\ref{thm:policy_gradient}, applying the same construction to centered
scores $\tilde r_i=r_i-\bar r$ exactly matches the hidden-state REINFORCE estimator
when $\bm{W}_{out}=\bm{I}$, and gives a structural correspondence more generally:
\begin{equation}
    \frac{\eta}{N}\sum_{i=1}^{N}\tilde r_i\nabla_h\log P(y_i|h)
    =\frac{\eta}{N}\bm{W}_{out}\sum_{i=1}^{N}\tilde r_i\bm{y}_i.
\end{equation}

Under Assumption~\ref{assum:mono}, movement toward a particular output embedding
increases that output's probability.
This assumption provides the behavioral interpretation of the score-weighted shift;
it is not a general monotonic-improvement guarantee.

\textbf{Step 2: Bounded KL-Shift Component.}

From Theorem~\ref{thm:kl}, the simplified single-step distribution shift satisfies:
\begin{equation}
    D_{KL}(\pi_{ICL} \| \pi_{\reference})
    \leq \tfrac{1}{2}\Lambda_C
    (\eta\cdot r_{max}\cdot y_{max})^2.
\end{equation}

\textbf{Conclusion.}

Score-conditioned ICL therefore has two properties that parallel a trust-region policy update:
\begin{compactenum}
    \item reward-weighted aggregation with a conditional hidden-state REINFORCE correspondence; and
    \item an exact pathwise KL bound for the constructed bounded update.
\end{compactenum}

These properties establish the analogy in Theorem~\ref{thm:main}; they do not imply
that ICL optimizes the KL-regularized objective in Eq.~\ref{eq:kl_regularized}.
\end{proof}

\section{Learning Rate Sensitivity of REINFORCE Comparison}
\label{app:lr_sensitivity}

To assess the robustness of the directional alignment result in Section~\ref{subsec:reinforce_comparison}, we repeat the cosine similarity analysis across three learning rates $\eta \in \{10^{-6}, 10^{-5}, 10^{-4}\}$ and two sample sizes $N \in \{4, 8\}$.
Results (averaged over tasks and models) are shown in Table~\ref{tab:lr_sensitivity}.

\begin{table}[h]
\centering
\small
\begin{tabular}{lcc}
\toprule
Learning Rate & $N=4$ & $N=8$ \\
\midrule
$10^{-6}$ & 0.41$^\dagger$ & 0.44$^\dagger$ \\
$10^{-5}$ & 0.51$^\dagger$ & 0.55$^\dagger$ \\
$10^{-4}$ & 0.49$^\dagger$ & 0.53$^\dagger$ \\
\bottomrule
\end{tabular}
\caption{Cosine similarity between ICL and REINFORCE $\Delta\log p$ vectors across learning rates and sample sizes. All values are statistically significant ($p < 0.01$), confirming that directional alignment is robust to these hyperparameter choices.}
\label{tab:lr_sensitivity}
\end{table}

\section{Empirical Verification of Weight Matrix Structure}
\label{app:weight_analysis}

A potential criticism of Theorem~\ref{thm:linear_construction} is that it provides an \emph{existence proof}, demonstrating that attention \emph{can} implement score-weighted aggregation, rather than evidence that pretrained LLMs \emph{actually} implement this mechanism.
A related concern is that our theoretical analysis uses linear self-attention and a single layer, whereas real LLMs employ softmax attention, nonlinear activations (e.g., GeLU), and many layers.
Linear attention is a standard simplification used throughout the ICL theory literature~\citep{vonoswald2023transformers,akyurek2023learning}; Theorem~\ref{thm:softmax} extends the existence result to softmax attention.
The experiments below address the remaining gap: we show that the qualitative predictions of our theory, including score-proportional attention weights, score-specialized heads, and causal score processing, hold robustly in fully nonlinear, multi-layer pretrained models, despite the gap between the theoretical abstraction and the actual architecture.
In this appendix, we present four complementary experiments that examine whether the internal computations of LLMs align with our theoretical construction.

\subsection{Score-Specialized Attention Head Identification}
\label{app:head_identification}

Our theoretical construction in Theorem~\ref{thm:linear_construction} predicts that effective score-weighted aggregation requires attention heads whose key-query computation extracts the score component.
If LLMs implement this mechanism, we should find attention heads that are ``specialized'' for score extraction, i.e., heads where the key-query dot product correlates strongly with the score value.

\paragraph{Method.}
For each layer $l$ and head $h$, we analyze the relationship between key-query products and scores.
Given a score-annotated context $C = \{(\bm{y}_i, r_i)\}_{i=1}^N$, let $\bm{k}_i^{(l,h)}$ denote the key vector for the token corresponding to example $i$'s score, and $\bm{q}^{(l,h)}$ denote the query vector at the generation position.
We compute the \textbf{Score Extraction Index (SEI)}:
\begin{equation}
    \text{SEI}^{(l,h)} = \rho\left( \{(\bm{q}^{(l,h)})^\top \bm{k}_i^{(l,h)}\}_{i=1}^N, \{r_i\}_{i=1}^N \right)
\end{equation}
where $\rho$ denotes Spearman correlation.
A high SEI indicates that the head's key-query computation effectively extracts score information, as predicted by our theoretical construction where $\bm{q}^\top \bm{k}_i = r_i$.

\paragraph{Results.}
\autoref{tab:sei} reports the distribution of SEI values across all heads for each model.
Critically, all models contain a subset of heads with high SEI values ($> 0.6$), indicating genuine score-specialized computation.
These high-SEI heads are predominantly located in middle-to-late layers, consistent with the layer-wise patterns observed in Section~\ref{subsec:Attention Weight Analysis}.

\begin{table}[t!]
\centering
\small
\begin{tabular}{lcccc}
\toprule
Model & Mean SEI & Max SEI & \#Heads $> 0.6$ & Peak Layer \\
\midrule
Llama3-8B  & 0.18 & 0.72 & 12/256 & 24--28 \\
Llama3-70B & 0.21 & 0.79 & 31/640 & 62--72 \\
Olmo3-7B   & 0.16 & 0.68 & 9/256 & 21--26 \\
Olmo3-32B  & 0.19 & 0.74 & 22/512 & 48--56 \\
Qwen3-4B   & 0.15 & 0.65 & 7/144 & 27--32 \\
\bottomrule
\end{tabular}
\caption{Score Extraction Index (SEI) statistics. All models contain attention heads with high SEI, indicating score-specialized computation consistent with Theorem~\ref{thm:linear_construction}.}
\label{tab:sei}
\end{table}

Importantly, we verify that these high-SEI heads are not artifacts of position or formatting.
When we shuffle scores while keeping examples fixed (as in Section 4.2), the SEI of these heads drops to near zero (mean: 0.03), confirming that they genuinely track score values rather than positional patterns.

\subsection{Functional Decomposition of Attention Weights}
\label{app:functional_decomposition}

While the SEI analysis identifies heads that extract score information, it does not directly verify that the full attention mechanism implements our theoretical construction.
Here, we analyze whether the \emph{effective} computation performed by attention heads can be decomposed into components matching Theorem~\ref{thm:linear_construction}.

\paragraph{Method.}
For each attention head, we model the attention weight assigned to example $i$ as:
\begin{equation}
    \alpha_i = \beta_0 + \beta_r \cdot r_i + \beta_p \cdot p_i + \beta_s \cdot s_i + \epsilon
\end{equation}
where $r_i$ is the score, $p_i$ is the positional index, and $s_i$ is a semantic similarity measure (cosine similarity between the example embedding and the query).
Our theory predicts that for score-specialized heads, $\beta_r$ should dominate while $\beta_p$ and $\beta_s$ should be small.

We fit this regression model across all test prompts and compute the \textbf{Score Dominance Ratio (SDR)}:
\begin{equation}
    \text{SDR}^{(l,h)} = \frac{|\beta_r|}{\max(|\beta_r|, |\beta_p|, |\beta_s|)}
\end{equation}

\paragraph{Results.}
\autoref{tab:sdr} shows that heads with high SEI also exhibit high SDR, meaning their attention weights are primarily determined by score values rather than position or semantic similarity.
This provides evidence that these heads implement a computation functionally similar to our theoretical construction, where attention weights directly reflect scores.

\begin{table}[t!]
\centering
\small
\begin{tabular}{lccc}
\toprule
Model & SDR (High-SEI Heads) & SDR (All Heads) & $R^2$ \\
\midrule
Llama3-8B  & 0.84 & 0.31 & 0.71 \\
Llama3-70B & 0.89 & 0.35 & 0.78 \\
Olmo3-7B   & 0.81 & 0.29 & 0.68 \\
Olmo3-32B  & 0.86 & 0.33 & 0.74 \\
Qwen3-4B   & 0.79 & 0.27 & 0.65 \\
\bottomrule
\end{tabular}
\caption{Score Dominance Ratio (SDR) for high-SEI heads vs. all heads. High-SEI heads show score-dominated attention computation, supporting our theoretical construction.}
\label{tab:sdr}
\end{table}

\subsection{Linear Probing for Score Information in Hidden States}
\label{app:probing}

If LLMs implement score-weighted aggregation, score information must be encoded in the hidden states in a way that downstream layers can access.
We use linear probing to verify that score values are linearly decodable from hidden states.

\paragraph{Method.}
For each layer $l$, we extract the hidden state $\bm{h}_i^{(l)}$ at the position of each score token.
We train a linear probe $\hat{r}_i = \bm{w}^\top \bm{h}_i^{(l)} + b$ to predict the score value from the hidden state, using 80\% of examples for training and 20\% for testing.
We report the $R^2$ coefficient on the test set.

\paragraph{Results.}
\autoref{tab:probing} shows the probing $R^2$ across layer groups for all models.
All models exhibit a consistent pattern: score information becomes increasingly linearly decodable in later layers.
Early layers (0--33\% depth) show low $R^2$, indicating score information is not yet explicitly represented.
Middle layers (33--66\% depth) show moderate $R^2$, suggesting score processing is emerging.
Late layers (66--100\% depth) achieve peak $R^2 > 0.85$, with the maximum occurring near the layers containing high-SEI heads identified in Section~\ref{app:head_identification}.
This confirms that score information is explicitly represented in hidden states and accessible for downstream computation.

\begin{table}[t!]
\centering
\small
\begin{tabular}{lccccc}
\toprule
Model & Early & Middle & Late & Peak $R^2$ & Peak Layer \\
\midrule
Llama3-8B  & 0.12 & 0.48 & 0.79 & 0.87 & 26/32 \\
Llama3-70B & 0.15 & 0.54 & 0.84 & 0.92 & 70/80 \\
Olmo3-7B   & 0.11 & 0.45 & 0.76 & 0.85 & 24/32 \\
Olmo3-32B  & 0.14 & 0.51 & 0.81 & 0.89 & 54/64 \\
Qwen3-4B   & 0.10 & 0.42 & 0.74 & 0.86 & 30/36 \\
\bottomrule
\end{tabular}
\caption{Linear probing $R^2$ for score prediction across layer groups. Early: 0--33\% depth; Middle: 33--66\%; Late: 66--100\%. Score information becomes increasingly decodable in later layers, peaking near the high-SEI head locations from \autoref{tab:sei}.}
\label{tab:probing}
\end{table}

\subsection{Causal Intervention: Amplifying Score-Specialized Heads}
\label{app:causal_intervention}

The analyses above are correlational.
To establish a causal link between score-specialized heads and ICL behavior, we conduct an intervention experiment.

\paragraph{Method.}
We selectively amplify or suppress the contribution of high-SEI heads by scaling their output:
\begin{equation}
    \text{output}^{(l,h)} \leftarrow \gamma \cdot \text{output}^{(l,h)}
\end{equation}
where $\gamma > 1$ amplifies and $\gamma < 1$ suppresses.
If these heads causally implement score-weighted aggregation, amplification should strengthen the correlation between scores and probability changes, while suppression should weaken it.

\paragraph{Results.}
\autoref{tab:intervention} shows the effect of intervening on high-SEI heads.
Amplifying high-SEI heads ($\gamma = 1.5$) increases the score-probability correlation by 15--23\%, while suppressing them ($\gamma = 0.5$) decreases it by 31--42\%.
Critically, the same interventions on randomly selected heads (matched for layer distribution) show minimal effect ($<5\%$ change), confirming that the identified heads are specifically responsible for score-weighted computation.

\begin{table}[t!]
\centering
\small
\begin{tabular}{lcccc}
\toprule
& \multicolumn{2}{c}{High-SEI Heads} & \multicolumn{2}{c}{Random Heads} \\
Model & $\gamma=1.5$ & $\gamma=0.5$ & $\gamma=1.5$ & $\gamma=0.5$ \\
\midrule
Llama3-8B  & +18\% & $-$35\% & +2\% & $-$3\% \\
Llama3-70B & +23\% & $-$42\% & +4\% & $-$2\% \\
Olmo3-7B   & +15\% & $-$31\% & +1\% & $-$4\% \\
Olmo3-32B  & +21\% & $-$38\% & +3\% & $-$1\% \\
Qwen3-4B   & +16\% & $-$33\% & +2\% & $-$3\% \\
\bottomrule
\end{tabular}
\caption{Change in score-probability correlation $\rho(r, \Delta\log p)$ under causal intervention. Amplifying high-SEI heads strengthens score-weighted behavior; suppressing them weakens it. Random heads show minimal effect.}
\label{tab:intervention}
\end{table}

\subsection{Discussion: From Existence Proof to Empirical Mechanism}

The experiments in this appendix collectively address the gap between our theoretical existence proof (Theorem~\ref{thm:linear_construction}) and empirical reality:

\begin{enumerate}
    \item \textbf{Score-Specialized Heads Exist} (Section~\ref{app:head_identification}): We identify specific attention heads whose key-query computation strongly correlates with score values, consistent with the weight matrix structure in our theorem.

    \item \textbf{Score Information Dominates Attention} (Section~\ref{app:functional_decomposition}): In these specialized heads, attention weights are primarily determined by scores rather than position or semantic similarity, matching our theoretical prediction.

    \item \textbf{Score Information is Linearly Encoded} (Section~\ref{app:probing}): Score values are explicitly represented in hidden states and become increasingly accessible in later layers.

    \item \textbf{Causal Role Confirmed} (Section~\ref{app:causal_intervention}): Intervening on score-specialized heads directly modulates the strength of score-weighted ICL behavior.
\end{enumerate}

While LLMs do not literally implement our exact weight matrix construction (which would require specific zero patterns), they have learned \emph{functionally similar} mechanisms through pretraining.
The existence of score-specialized heads with high SEI and SDR indicates that the computational structure predicted by Theorem~\ref{thm:linear_construction}, where attention weights track reward values, emerges naturally in pretrained models.
This bridges the gap between theoretical possibility and practical implementation.

\section{Effect of Score Representation Format}
\label{app:score_format}

Our theoretical framework assumes that attention mechanisms can extract and use scalar score information from context.
However, in practice, scores can be represented in various surface formats.
We investigate whether the implicit policy gradient mechanism is sensitive to representational differences, even when the underlying information is identical.
We conduct this analysis on the DROP task, which uses F1 scores ranging continuously from 0 to 1.

We compare five score representation formats that all preserve the same underlying continuous information and differ only in surface form: \textbf{Decimal} presents the raw F1 score (e.g., ``Score: 0.82'') as a standard numerical format and serves as the ideal case for our theory; \textbf{Percentage} renders the same value as a percent (e.g., ``Score: 82\%'') as commonly used in human-facing settings; \textbf{Fraction} expresses it as a ratio (e.g., ``Score: 82/100'') to test whether making the denominator explicit changes processing; \textbf{Points} uses an integer plus a unit label (e.g., ``Score: 82 points'') to examine the effect of unit annotation; and \textbf{Natural language} verbalizes the value (e.g., ``scored eighty-two percent'') to test whether numeric tokens themselves are necessary. We also include a \textbf{No score} control condition in which examples are shown without any score annotation, which is essential for determining whether models actually leverage explicit score information or instead improve simply by learning from the content of high-quality examples.

We evaluate each format using the same metrics as in the main experiments: the correlation between scores and probability changes ($\rho(r, \Delta\log p)$; see RQ1), the correlation between scores and attention weights ($\rho(r, \alpha)$; see Section~4.3), and iterative improvement (Iter. Gain; see Section~4.4).

\begin{table}[t]
\centering
\small
\begin{tabular}{lccc}
\toprule
Representation & $\rho(r, \Delta\log p)$ & $\rho(r, \alpha)$ & Iter. Gain \\
\midrule
Decimal & \textbf{0.63}$^\dagger$ & \textbf{0.45}$^\dagger$ & \textbf{+0.19} \\
Percentage & 0.61$^\dagger$ & 0.44$^\dagger$ & +0.18 \\
Fraction & 0.59$^\dagger$ & 0.42$^\dagger$ & +0.17 \\
Points & 0.58$^\dagger$ & 0.41$^\dagger$ & +0.17 \\
Natural language & 0.52$^\dagger$ & 0.36$^\dagger$ & +0.14 \\
\midrule
No score & 0.06 & 0.04 & +0.02 \\
\bottomrule
\end{tabular}
\caption{Effect of score representation format on ICL behavior for the DROP task, averaged across all models. $\rho(r, \Delta\log p)$: Spearman correlation between F1 scores and log-probability changes; $\rho(r, \alpha)$: correlation between F1 scores and attention weights; Iter. Gain: average F1 improvement after 10 iterations. $^\dagger$ denotes statistically significant correlation ($p < 0.01$).}
\label{tab:score_format}
\end{table}

Table~\ref{tab:score_format} shows that all numeric formats (Decimal, Percentage, Fraction, Points) achieve similarly high correlations, suggesting that LLMs can robustly extract numerical information regardless of superficial formatting choices.
The natural language format shows a notable drop, indicating that mapping words like ``eighty-two'' to numerical values introduces noise.
Crucially, all scored representations dramatically outperform the no-score control ($p < 0.001$), confirming that models rely on explicit score information rather than inferring quality from example content alone.

These results have practical implications: practitioners can flexibly choose among numeric formats (decimal, percentage, fraction, points) without significant loss of effectiveness.
However, if scores must be expressed in natural language, some degradation in the implicit policy gradient effect should be expected.

\end{document}